\documentclass[letterpaper]{article} 

\usepackage{multicol}

\newcommand{\tbf}[1]{\textbf{#1}\xspace}

\newcommand{\epp}{\ensuremath{\operatorname{EPPT}}\xspace}
\newcommand{\bmc}{\ensuremath{\operatorname{BMCP}}\xspace}
\newcommand{\rdr}{\ensuremath{\mathsf{RAS}}\xspace}

\newcommand{\sgreedy}{\ensuremath{\mathsf{Greedy}}\xspace}
\newcommand{\srandom}{\ensuremath{\mathsf{Uniform}}\xspace}

\usepackage[multiple]{footmisc}
\usepackage[group-separator={,}]{siunitx}

\usepackage{macros_pan}

\usepackage{aaai23}
\usepackage{times}  
\usepackage{helvet}  
\usepackage{courier}  
\usepackage[hyphens]{url}
\usepackage{graphicx}  
\usepackage{natbib}  
\usepackage{caption} 
\usepackage{subcaption}
\usepackage{multirow}
\usepackage{xcolor}
\usepackage{xspace}
\usepackage[switch]{lineno}
\newcounter{int}
\makeatletter
\newcommand{\citen}[1] {\setcounter{int}{0}\@for\tmp:=#1\do{%
\ifnum \value{int}>0; \fi%
\setcounter{int}{1}%
\citeauthor{\tmp} \shortcite{\tmp}}}
\makeatother
\makeatletter
\newcommand{\citenp}[1]{\setcounter{int}{0}\@for\tmp:=#1\do{%
\ifnum \value{int}>0; \fi%
\setcounter{int}{1}%
\citeauthor{\tmp}, \citeyear{\tmp}}}
\makeatother

\begin{document}

\title{Equity Promotion in Public Transportation}
\author {
    Anik Pramanik\textsuperscript{\rm 1}\equalcontrib,
    Pan Xu\textsuperscript{\rm 1}\equalcontrib,
    Yifan Xu\textsuperscript{\rm 2}\equalcontrib
}
\affiliations {
    \textsuperscript{\rm 1} Department of Computer Science, New Jersey Institute of Technology, Newark, USA\\
    \textsuperscript{\rm 2} School of Cyber Science and Engineering, Southeast University, Nanjing, China\\
    ap2645@njit.edu,
    pxu@njit.edu,
    xyf@seu.edu.cn
}

\maketitle

\begin{abstract}
There are many news articles reporting the obstacles confronting poverty-stricken households in access to public transits. These barriers create a great deal of inconveniences  for these impoverished families and more importantly, they contribute a lot of social inequalities. A typical approach addressing the issue is to build more transport infrastructure to offer more opportunities to access the public transits especially for those deprived communities. Examples include adding more bus lines connecting needy residents to railways systems and extending existing bus lines to areas with low socioeconomic status. Recently, a new strategy is proposed, which is to harness the ubiquitous ride-hailing services to connect disadvantaged households with the nearest public transportations. Compared with the former infrastructure-based solution, the ride-hailing-based strategy enjoys a few exclusive benefits such as higher effectiveness and more flexibility.

In this paper, we propose an optimization model to study how to integrate the  two approaches together for equity-promotion purposes. Specifically, we aim to design a strategy of allocating a given limited budget to different candidate programs such that  the overall social equity is maximized, which is defined as the minimum covering ratio among all pre-specified protected groups of households (based on race, income, etc.). We have designed a linear-programming (LP) based rounding algorithm, which proves to achieve an optimal  approximation ratio of $1-1/e$. Additionally, we test our algorithm against a few baselines on real data assembled by outsourcing multiple public datasets collected in the city of Chicago. Experimental results confirm our theoretical predictions and demonstrate the effectiveness of our LP-based strategy in promoting social equity, especially when the budget is insufficient.  
\end{abstract}




\section{Introduction}
We consider the last-mile problem in public transportation. There are roughly 20\% of households that are at or below the federal  poverty line lack access to a car, and the percentage can get as high as 33\% among the low-income African and Latino population~\cite{berube2006socioeconomic}. For these impoverished families: On the one hand, they rely on the 
public transportation as the only travel option for daily activities such commuting to work, shopping, \etc; on the other hand, many of them live relatively far away from public transits and need to walk a long distance to even the nearest bus and/or metro stops. The difficulty in access to public transport creates a great deal of trouble and inconvenience for these poverty-stricken families and contributes to many social inequalities {~\cite{degood2016can}.} One recent example is the radically unequal access to COVID-19 vaccines in the early stage of rollout in 2021, where it had been widely reported that vaccination rate of Black People was greatly lagging behind that of White counterparts, and one of the main causes was the obstacle in access to public health providers{~\cite{wp-1,nyt-2,los-1,wp-2}.}

The traditional way to improve the access to public transits includes creating more bus lines connecting needy residents to railway systems, extending existing bus lines to deprived communities, and increasing the frequency and service hours of transit providers, to name a few. Recently, local officials are weighing the option of integrating the popular ride-hailing services offered by Uber and Lyft to the existing toolkit to combat the last-mile problem~\cite{jin2019uber,kong2020does}. One example is to establish an auxiliary subsidized program to allocate reimbursed ride-hailing trips to needy households to help them access to the near transit stations. Compared with the traditional solutions, ride-hailing services enjoy benefits such as more flexibility and mobility (can be requested via mobile apps upon needed) and a lower cost in general.  Additionally, ride-hailing-based program can help effectively target needy households, especially when they sparsely and remotely scatter over a large area where traditional means are either cost prohibitive or ineffective. 

There are a few challenging issues  in the last-mile problem, including  how to craft eligibility to identify the set of qualified needy households, how to set subsidizing guidelines for ride-hailing services for people under different spatial and financial conditions, how to design routes and schedules of new bus lines, and how to split  limited budget to different approaches. For policy-related issues, existing literatures have already proposed a few answers. For example,~\citet{degood2016can} outlined eligibility spatial criteria among other financial factors as follows: walking distance to the nearest bus lines should fall between $[0.25, 3.5]$ miles and to the nearest railway station between $[0.5, 3.5]$ miles.\footnote{Here upper bounds are set to ensure a high utilization of limited funds, which can benefit more people that live relatively closer to the existing public system than those far away.} They also proposed detailed suggestions on setting tailored subsidizing policy for needy households  based on the income level and the household size.

In this paper, we consider a non-policy-related question of optimizing budget allocation: Given a limited budget, a set of established qualified households, and a collection of well-defined promotion programs  (\eg a set of new bus lines with full information regarding the operational cost, routes and schedules, a comprehensive subsidizing framework of ride-hailing trips), how to design a best strategy of allocating  the limited fund to different proposed programs to maximize the social equity? There are several different metrics for the social equity; one example is the min coverage ratio among all pre-specified protected groups based on sensitive information like race and ethnicity, which is commonly used in fairness- and equity-related promotions{~\cite{ma2020group,nanda2020balancing,Hosseini2022ClassFI}. }

We formalize our problem, called  \textbf{Equity Promotion in Public Transportation} (\epp), as follows. Suppose we have a set $I$ of qualified needy households, where each $i$ has the following three kinds of information: (1) spatial attributes such as the living location, walking distances to the nearest bus lines and the nearest railway station; (2) financial factors including income level and household size, and (3) basic demographics such as race and ethnicity. Assume we are offered a total fund $B$ for a  given time window (\eg one year or one quarter). We have a collection $J$ of candidate bus lines to open, where each bus line $j$ is characterized as an operational cost $c_j$ and a subset $S_j \subseteq I$ denoting the set of target households covered,\footnote{Here we view bus lines with the same route but different schedules and/or frequencies as distinct  since they can incur different costs and cover varied-sized sets of needy households.} and where $\{S_j| j \in J\}$ can be possibly overlapping. Each target household $i \in I$ can be covered either by a candidate bus line $j \in J$ if $S_j \ni i$ or by being enrolled into the ride-hailing-based welfare program, where the latter will induce a cost $c_i$ determined by $i$'s spatial and financial traits. Suppose we have a collection of protected groups $\cG=\{g\}$, where each group $g\subseteq I$ is a subset of target households sharing specific demographics (\eg Black, White, and Asian). Note that all information of $\{I, J, B, \{c_i, c_j, S_j| i \in I, j\in J\}, \cG=\{g\} \}$ is given as part of the input. In order to better expose the technical challenges in our problem, we present two versions of a budget allocation strategy, as follows.

\xhdr{Deterministic version of a budget allocation strategy}. For each bus line $j \in J$, let  $x_j=1$ indicate that $j$ is to open; for each qualified household $i \in I$, let  $x_i=1$ indicate that $i$ should be enrolled into the ride-hailing program. Thus, a deterministic budget allocation plan can be captured as a binary vector $\x \in \{0,1\}^{|J|+|I|}$, and it is called \emph{feasible} if $\sum_{j \in J} c_j x_j+\sum_{i \in I} c_i x_i \le B$, \ie the total cost is within the given budget.  For each needy household $i \in I$, let $y_i:=\min(1, x_i+\sum_{j: S_j \ni i}x_j)$, which indicates if $i$ is covered ($y_i=1$) or not.  Observe that $y_i=0$ ($i$ is not covered)  iff $x_i=x_j=0$ for all $j$ with $S_j \ni i$, which means $i$ is not added to the ride-hailing program, either none of bus lines covering $i$ is open. The coverage ratio of group $g$ on $\x$ can be then expressed as $\sum_{i \in g}y_i/|g|$, where $|g|$ refers to the cardinality of group $g$. The resulting objective of the social equity can be computed as $\min_{g \in \cG} \sum_{i \in g}y_i/|g|$. 

\xhdr{Randomized version of a budget allocation strategy}. Let $\Phi=\{\phi_k|1 \le k \le K\}$ be a collection of all possible feasible deterministic strategies as described above.\footnote{Note that the number $K$ of all possible feasible strategies is upper bounded by $2^{|I|+|J| }$.} A randomized budget allocation strategy $\ALG$ then can be captured as a distribution $\cD$ over $\Phi$ such that $\ALG$ will sample and run a strategy from $\Phi$ following distribution $\cD$. For each household $i \in I$, let random variable $Y_i=1$ indicate that $i$ is covered in $\ALG$ (and $Y_i=0$ otherwise). For each group $g \in \cG$, the \emph{expected} coverage ratio under $\ALG$ then can be expressed as $\E[\sum_{i \in g} Y_i/|g|]$, where the expectation is taken over $\cD$, the random choice taken by $\ALG$. The final objective of the social equity can be computed as $\min_{g \in \cG} \E[\sum_{i \in g}Y_i/|g|]$. 

\xhdr{Deterministic vs.\ randomized strategies}. (1) Note that any deterministic strategy can be cast as a special randomized one. This suggests that by expanding the focus from deterministic to randomized, we (the algorithm designer) can \emph{potentially} land
at a better strategy in terms of achieving a larger objective value, while how much extra value we can gain  really depends on the specific problem. As Example~\ref{exam:a} shows, some randomized strategy can far outperform the optimal deterministic on some instances of our model. (2) Randomized strategies make more sense in the context of promoting equity in public transits compared with deterministic ones. In most practical scenarios, local governments can secure only a limited budget that can cover a small portion of impoverished households. In this case, any deterministic budget-allocation strategy could benefit only a small fixed set of targets and inevitably leave far more vulnerable households unaffected, which could cause more social injustices. In contrast, randomized strategies can potentially impact a far larger deprived population. Also, they are easy to implement in practice, \eg  by randomly sampling a set of targets every time and then recruiting them to the ride-hailing program, and/or by alternatively operating different bus lines based on monthly/quarterly schedules.

\section{Preliminaries and Main Contributions}
  Throughout this paper, we denote $[n]=\{1,2,\ldots, n\}$ for a generic positive integer $n$; we use \OPT to denote both of an optimal strategy and the corresponding performance, and the same for \ALG, which denotes both a generic algorithm and its performance.

  \xhdr{Approximation ratio} (AR). For NP-hard combinatorial optimization problems, a powerful framework is called \emph{approximation algorithms}, where we aim to design an efficient algorithm (polynomial running time) with a guaranteed performance from the optimal. Consider a maximization problem like \epp as studied here. Let $\ALG$ be an approximation algorithm (possibly randomized) and $\OPT$ denote an optimal algorithm with no running-time constraint and its performance. We say $\ALG$ achieves an approximation ratio of $\rho \in [0,1]$ if $\E[\ALG] \ge \rho \cdot \OPT$ for all possible input instances.
  
  \xhdr{Connection to Budgeted Maximum Coverage Problem (BMCP)}. To the best of our knowledge, BMCP is the model closest to ours, which is a generalization of the classical Maximum Coverage Problem. The basic setting is as follows. We have a ground set of $I$ and a collection $\cS=\{S_j| j \in J\}$ of subsets of $I$, where each subset $S_j$ indexed by $j \in J$  is associated with a cost $c_j>0$. Suppose we are given a total budget $B$, and we aim to identify a sub-collection, denoted by $J' \subseteq J$, to maximize the total coverage $|\cup_{j \in J'} S_j|$ subject to the budget constraint, \ie $\sum_{j \in J'} c_j \le B$. BMCP and its related variants are well studied in theoretical computer science community~\cite{cohen2008generalized,khuller1999budgeted}, and most of them can be solved via a greedy-based framework with an optimal approximation ratio of $1-1/\sfe \sim 0.632$~\cite{feige1998threshold}. \emph{Note that \bmc can be cast as a special case of our problem \epp  when it has only one single protected group  and offers no any ride-hailing-based program, which suggests that \epp admits no approx-ratio better than $1-1/\sfe$ unless P=NP}.    The lemma below further highlights the difference between BMCP and \epp.

\begin{lemma}\label{lem:intro-1}
For \bmc, any optimal randomized strategy can be realized by a deterministic one. In contrast, there exists some instance of \epp where an optimal randomized strategy can strictly beat an optimal deterministic.
 \end{lemma}

The lemma above suggests that expanding the set of strategy choices from deterministic to randomized will offer no extra power for \bmc but will possibly do for \epp. The difference is mainly due to the fact that the objective function of \bmc is linear (thus, linearity of expectation can be applied), whereas that of \epp is nonlinear.\footnote{Actually, the first part of statement of Lemma~\ref{lem:intro-1} can be generalized to any optimization problem where the objective function can be expressed as a linear function of decision variables.}

\begin{proof}
We prove the first claim for $\bmc$. Consider a given instance of $\bmc$ with a ground set of $I$ and a given budget $B$. Let $\OPT_R$ be an optimal randomized strategy that is characterized by the following distribution over the collection $\Phi=\{\phi_k | k \in [K]\}$ of all feasible deterministic strategies: $\OPT_R$ will run $\phi_k$ with probability $q_k$ with $\sum_{k \in [K]} q_k=1$. For each element $i \in I$ and each deterministic strategy $k$, let $Y_{ik}=1$ indicate that $i$ is covered in $\phi_k$ and $Y_{ik}=0$ otherwise. Thus, $\sum_{i \in I} Y_{ik}$ captures the coverage of strategy $\phi_k$.
Observe that
\[
\E[\OPT_R]=\sum_{k \in [K]}\bp{\sum_{i \in I} Y_{ik}} \cdot q_k \le \max_{k \in [K]}\bp{\sum_{i \in I} Y_{ik}}=\OPT_D,
\]
where $\OPT_D$ denotes the performance of an optimal deterministic. The claim for \epp can be seen on Example~\ref{exam:a}.
\end{proof}

\begin{example}\label{exam:a}
Consider a toy example of \epp, where $\cG=\{g_1, g_2\}$ with $g_1=\{a\}$, $g_2=\{b\}$ and $I=\{a,b\}$. Let $J=\emptyset$ (no bus lines to open) and covering households $a$ and $b$  via the ride-hailing program each incur a unit cost that is equal to the budget, \ie $c_a=c_b=B=1$. We can verify that for any deterministic strategy can achieve an equity being 0 since it can cover only one household in one group; thus, $\OPT_D=0$. Consider such a randomized strategy that is to cover only $a$ or $b$ via the ride-hailing program  
each with probability $1/2$. Then it achieves an \emph{expected} equity of $1/2$, which suggests that $\E[\OPT_R] \ge 1/2>\OPT_D=0$. 
\end{example}



  

\subsection{Main Contributions and Techniques}
 In this paper, we consider  a technical issue of promoting the social equity by optimizing the budget allocation to different candidate welfare programs assisting deprived households in access to the public transits. In all technical sections, we assume WLOG  that the cost associated with each program is no more than $1$ with $B \ge 1$ by scaling down all costs such that $\max_{\ell \in I \cup J}c_\ell=1$. Our main technical result is stated as follows.

\begin{theorem}\label{thm:main-1}[Section~\ref{sec:alg}]
There exists a linear-programming (LP) based  rounding strategy  (\rdr) that achieves an optimal approximation ratio of $1-1/\sfe$ for \epp, which uses a budget no more than $B$ in expectation and no more than $B+1$ for any realization.
\end{theorem}

\noindent\tbf{Remarks on Theorem~\ref{thm:main-1}}.  (1) Note that our problem \epp captures \bmc as a special case; thus, no algorithm (including randomized) can achieve an approx-ratio better than $1-1/\sfe$, which suggests the optimality $\rdr$ in terms of approx-ratio. (2) As stated in Theorem~\ref{thm:main-1}, our strategy is feasible in budget under expectation, though no always. This is inconsequential and can be mitigated technically. Note that the total absolute overflow of the budget is $1$, which represents the max cost associated with any candidate programs. By decomposing or splitting a given program into several copies (\eg replacing a bus line operating 12 hours a day by 12 bus lines sharing the same route but operating one hour only a day), we can significantly  reduce the max cost incurred by any candidate welfare programs. Another note is that our input setting is positioned on a given time window (\eg one  quarter or a year). Thus, we expect to run our randomized strategy repeatedly stretching over a long period covering multiple units of the time window considered here. The fact that our strategy is feasible in the budget under expectation suggests its feasibility in \emph{the long run}, making it somewhat acceptable in practice.

We implement \rdr and compare it against several natural baselines on real data that is assembled by integrating multiple real datasets from different public sources.   Experimental results confirm our theoretical predictions and demonstrate the effectiveness of our LP-based randomized strategy (\rdr) in promoting social equity in public transportation, especially when the budget is small. This highlights the practical value of our proposed strategy in improving social equity since the lack of fund for the public infrastructure is widely reported and is prevalent across the USA; see, \eg \cite{li-1,li-2,li-3}. Furthermore, the results suggest the great complementary value brought by the ride-hailing-based program to the traditional bus-line-based when the budget is insufficient.





\xhdr{Technical challenges}. The main technical challenge in our problem is partially due to the non-linear objective function of maximizing the overall social equity that is quantified as the minimum (expected) coverage ratio among all protected groups.  One the one hand, as shown in  Example~\ref{exam:a}, randomized strategies can be substantially more powerful than deterministic ones on the objective studied here; on the other hand, the design and analysis of a randomized algorithm are  more technically challenging in general compared with that of a deterministic, requiring to craft and add appropriate random bits based on the  specific  input structure. Note that the introduction of ride-hailing welfare program does not add any new \emph{technical} challenges to the problem,\footnote{We can actually view covering a household $i$ by the ride-hailing program alternatively as a virtual new bus line $j'$ with cost $c_{j'}=c_i$ and coverage $S_{j'}=\{i\}$.} though it could mean a lot in practice in terms of promoting the social equity; see experimental results of the impact on the equity brought by the ride-hailing program in Section~\ref{sec:exp-res}.

One of the main technical contributions in the paper is to propose a \emph{weighted} version of {Dependent Rounding} (DR), where the original version of DR was introduced by~\citet{gandhi2006dependent} that is to round a vector in a fractional bipartite-matching polytope to an integral that is required to lie in the same polytope and satisfy a few properties. One specific feature imposed on the final rounded solution in~\cite{gandhi2006dependent} is that the \emph{unweighted} sum of variables incident to every vertex should remain invariant before and after rounding. In our case, we need to ensure the \emph{weighted} sum of all variables has a gap as small as possible between the original fractional and the final rounded integral solutions. To solve this challenge, we craft a more delicate weighted version of DR; see Algorithm~\ref{alg:rdr}. 
 




\xhdr{Other related work}. There are a few works that have considered resource allocation in online setting under different contexts; see, \eg equity promotion in vaccination~\cite{Xu2022EquityPI}, ride-hailing resource allocation~\cite{lesmana2019}, online resource-allocation problems with limited choices in the long-chain design~\cite{DBLP:journals/mansci/AsadpourWZ20}, and online fair division problem~\cite{aleksandrov2015online}.
Another research line has studied matching policy design and simulation in integrating the ride-hailing and public transits~\cite{Basu2018AutomatedMV,Boone2018MobilityAA,Shen2018IntegratingSA,Stiglic2018EnhancingUM,Yan2019IntegratingRS}.
However, these studies haven't considered the objective of promoting the overall social equity as here.
\section{An LP-based Rounding Algorithm}\label{sec:alg}
Recall that the introduction of ride-hailing-based program brings no technical challenge to \epp; see discussions in ``\tbf{Technical challenges}.'' For the ease of exposition, we consider an input instance of \epp with the bus-lines-based  program only by treating to cover each household via ride-hailing trips as a special one-one bus line covering the specific household only. Consider an optimal randomized strategy \OPT.\footnote{Throughout this paper, we use the two terms ``strategy'' and ``algorithm''  interchangeably.} For each target household $i \in I$, let $y_i$ be the probability that $i$ is covered in \OPT. For each candidate bus line $j \in J$, let $x_j$ be the probability that $j$ is set to open in \OPT. Consider the Linear Program (LP) below.

\begin{alignat}{2}
\max &~~\min_{g \in \cG}\bp{ \sum_{i \in g} y_j/|g|},&&  \label{obj-1} \\
 & \sum_{j \in J} c_j x_j \le B,  &&~~  \label{cons:b} \\ 
 &  y_i \le \min \bp{1, \sum_{j: S_j \ni i} x_j}, &&~~ \forall i \in I \label{cons:i} \\ 
  &0 \le  y_i, x_j \le 1 && ~~i \in I, j \in J. \label{cons:e}
\end{alignat}

 \begin{lemma}\label{lem:lp}
The optimal value of \LP~\eqref{obj-1} is a valid upper bound for the performance of an optimal randomized strategy.
\end{lemma}
 \begin{proof}
We can verify that Objective~\eqref{obj-1} captures the exact performance of \OPT (\ie the min expected coverage ratio among all protected groups). To prove our claim, it suffices to show that the strategy $\{x_j, y_i\}$ of \OPT is feasible to all constraints there. For each $j \in J$ and $i \in I$, let $X_j=1$ and $Y_i=1$ indicate $j$ is set to open and $i$ is covered in \OPT, respectively.  Thus, $\E[X_j]=x_j$ and $\E[Y_i]=y_i$ for every $j$
 and $i$. Since \OPT can be viewed as a certain randomization over all feasible deterministic strategies, $\sum_{j \in J} c_j X_j \le B$. Thus, $\E[\sum_{j \in J} c_j X_j] \le B$, which leads to Constraint~\eqref{cons:b}. Note that for each $i \in I$, $Y_i \le 1$ and $Y_i \le \sum_{j: S_j \ni i} X_j$. Taking expectation on both sides, we get Constraint~\eqref{cons:i}. Constraint~\eqref{cons:e} is valid since $\{x_i, y_j\}$ are all probability values.
  \end{proof}

\let\oldnl\nl
\newcommand{\nonl}{\renewcommand{\nl}{\let\nl\oldnl}}
\begin{algorithm*}[ht!] 
\DontPrintSemicolon
\nonl \textbf{Input}: An input instance of Equity Promotion in Public Transportation (\epp): $\{I,J,B, \{c_j, S_j| j \in J\}, \cG=\{g\}\}$.\;
\nonl \textbf{Output}: A \emph{randomized} budget-allocation strategy denoted by a binary vector $\X \in \{0,1\}^{|J|}$ with $X_j=1$ indicating to open bus line $j \in J$ and $X_j=0$ otherwise.\;
  \SetAlgoLined 
 Solve \LP~\eqref{obj-1} and let $\x^*=\{x^*_j|j \in J\}$ be part of the optimal solution. Set $\X=\x^*$, \ie $X_j=x_j^*$ for all $j \in J$.\;
\While{1}{
\uIf{there exists two fractional values in $\X$, say, $0<X_p, X_q<1$,}
{compute $\alp:=\max\{\ep>0: X_p+\ep \le 1, X_q-c_p \cdot \ep/c_q \ge 0\}$; $\beta:=\max\{\ep>0: X_p-\ep \ge 0, X_q+c_p \cdot \ep/c_q \le 1\}$;\;
with probability $\beta/(\alp+\beta)$, update $X_p \gets X_p+\alp$, $X_q \gets X_q-c_p \cdot \alp/c_q$;\;
with probability $\alp/(\alp+\beta)$, update $X_p \gets X_p-\beta$, $X_q \gets X_q+c_p \cdot \beta/c_q$.}
\uElseIf{there exists one single fractional value $0<X_j<1$ in $\X$,}
{with probability $X_j$, set $X_j \gets 1$; and with probability $1-X_j$, set $X_j \gets 0$.}
\Else{\tbf{Break}.}
}
 \caption{A Randomized Allocation Strategy (\rdr) for Equity Promotion in Public Transportation.} \label{alg:rdr}
\end{algorithm*}

Based on an optimal solution to \LP~\eqref{obj-1}, we design a strategy, denoted by \rdr (Algorithm~\ref{alg:rdr}), which is  built on randomized dependent rounding. Let $\x^*=(x_j^*)$ be part of the optimal solution of \LP~\eqref{obj-1}. The main idea is to apply a series of rounding procedures to transform $\x^*$ to a random binary vector $\X^* \in \{0,1\}^{|J|}$ such that the expected total cost $\sum_{j \in J} c_j \E[X^*_j]$ is as small as possible (ideally no more than $B$), while each household $i \in I$ can get covered with a probability as large as possible.  


The overall picture of our rounding procedure is as follows. During each step, we  identify two remaining fractional values in $\x^*$ if any, and ``twist'' the two values together following one of the two directions randomly, and in each direction, one fractional value will be rounded up and the other will be down. We carefully design the rounding procedure such that it has Properties (\tbf{P1}), (\tbf{P2}), (\tbf{P3}), and  (\tbf{P4}) as outlined in Lemma~\ref{lem:pr}, which is vital to the performance of the randomized strategy $\rdr$ in Algorithm~\ref{alg:rdr}. Generally, (\tbf{P1}) suggests that at least one fractional value in $\x^*$ will be rounded in each step and thus, $\rdr$ will terminate after at most $|J|$ steps.\footnote{More precisely, \rdr will stop after  at most $\lceil |J|/2 \rceil$ steps.}; (\tbf{P2}) says that the expectation on each variable remains the same throughout the rounding process;  (\tbf{P3}) implies that the total budget on the final rounded solution could get an overflow by at most $1$; and   (\tbf{P4}) indicates \emph{negative} correlation over any subset of variables.


\begin{lemma}\label{lem:pr}
\rdr in Algorithm~\ref{alg:rdr} satisfies the following properties in each rounding step:  (\tbf{P1}) At least one fractional value is rounded to either 0 or 1; (\tbf{P2}) The expectation on each value keeps invariant after rounding; (\tbf{P3}) The total cost remains the same if two fractional values are involved and will get increased  by at most $1$ if only one value gets involved; (\tbf{P4}) For any subset $\cS \subseteq J$, $\E[\prod_{j \in \cS} (1-X_j)]$ will never get increased.
\end{lemma}

We present the proof of  (\tbf{P1}), (\tbf{P2}), and (\tbf{P3}) here and leave that of (\tbf{P4}) to Appendix.
\begin{proof}
Consider two cases. \tbf{Case 1}: There exists two fractional values in $\X$, say $0<X_p,X_q<1$. Following the definition of $\alp$ and $\beta$, we see that the two updating procedures,  $X_p \gets X_p+\alp$, $X_q \gets X_q-c_p \cdot \alp/c_q$ and $X_p \gets X_p-\beta$, $X_q \gets X_q+c_p \cdot \beta/c_q$, both will end up with at least one integral value on either $X_p$ or $X_q$. Let $X'_p$ and $X'_q$ be the rounded value in the end. We see
\begin{align*}
\E[X'_p|X_p, X_q]&=(X_p+\alp)\cdot \frac{\beta}{\alp+\beta}+ (X_p-\beta) \cdot \frac{\alp}{\alp+\beta}=X_p;\\ 
\E[X'_q|X_p, X_q]&=\big(X_q -\frac{c_p\alp}{c_q}\big)\cdot \frac{\beta}{\alp+\beta}+ \big(X_q+\frac{c_p \beta}{c_q}\big)\cdot \frac{\alp}{\alp+\beta}\\
&=X_q. 
\end{align*}
Thus, we conclude that the expectations of the two values both remain unchanged after rounding. Observe that for each updating procedure, the total cost remains the same:
\begin{align*}
&(X_p+\alp) \cdot c_p+(X_q -\frac{c_p\alp}{c_q}\big) \cdot c_q=X_p \cdot c_p+X_q\cdot c_q;\\
&(X_p-\beta) \cdot c_p+(X_q +\frac{c_p\beta}{c_q}\big) \cdot c_q=X_p \cdot c_p+X_q\cdot c_q.
\end{align*}

Now consider \tbf{Case 2}: There exists one single fractional values in $\X$, say $0<X_j<1$.
Let $X'_j$ be the value after being rounded. According to our procedure, we see that $X'_j \in \{0,1\}$. Furthermore, $\E[X'_j|X_j]=X_j \cdot 1=X_j$. Note that in this case the total cost will get increased by at most $c_j \le 1$ (recalled that we assume WLOG that $c_j \le 1$ for all $j \in J$).  
\end{proof}

We are ready to prove the main Theorem~\ref{thm:main-1}. For ease of exposition, we restate Theorem~\ref{thm:main-1} in the lemma below.

\begin{lemma}\label{lem:bug}
For the randomized strategy $\rdr$, (1) it uses a total budget no more than $B$ in expectation; (2) it uses a total budget no more than $B+1$ for any realization; and (3) it achieves an approx-ratio at least of $1-1/\sfe$.
\end{lemma}
\begin{proof}
Let $\{x^*_j, y_i^*\}$ be an optimal solution to LP~\eqref{obj-1}. We prove Claim (1) first. Let $\X^*=(X_j^*)$ the final rounded binary vector output by \rdr. Property (\tbf{P2}) in Lemma~\ref{lem:pr} suggests that $\E[X_j^*]=x_j^*$ for every $j \in J$ since we initialize $\X=\x^*=(x_j^*)$. Thus, the expected total cost on $\X^*$ should satisfy $
\E[\sum_{j \in J} c_j X_j^*]=\sum_{j \in J}c_j x_j^* \le B$, where the last inequality is due to Constraint~\eqref{cons:b}. For Claim (2): Note that the case when there is only one fractional value involved in the rounding could happen at most once in $\rdr$, which implies that the total cost can get inflated by at most $1$.

Now we show Claim (3). For each $i \in I$, let $Y_i^*=1$ indicate that $i$ is covered in the final strategy $\X^*$ and $Y_i^*=0$ otherwise. Let $\cS_i=\{j: S_j\ni i\}$, which denotes the set of indices of bus lines that cover $i$. We can verify that $Y_i^*=1-\prod_{j \in \cS_i}(1-X_j^*)$, and  $\E[\prod_{j \in \cS_i}(1-X_j^*)] \le \prod_{j \in \cS_i}(1-x_j^*)$ (by repeatedly applying (\tbf{P4}) in Lemma~\ref{lem:pr}). Thus, 
\begin{align}
\E[Y_i^*]&=1-\E\bb{\prod_{j \in \cS_i}(1-X_j^*)} \nonumber \\
&\ge 1-\prod_{j: j \in \cS_i}(1-x_j^*)~~ \big(\mbox{(\tbf{P4}) in Lemma~\ref{lem:pr}}\big)  \nonumber\\
& \ge 1-\prod_{j: j \in \cS_i} \sfe^{-x_j^*}=1- \sfe^{-\sum_{j \in \cS_i}x_j^*}. \label{ineq:ys}
\end{align}

Note that by Constraint~\eqref{cons:i}, $y_i^* \le \min \big(1, \sum_{j \in \cS_i}x^*_j\big)$.
Consider these two cases.  \tbf{Case (1)}: $\sum_{j \in \cS_i}x_j^* \ge 1$. Then 
\[
\E[Y_i^*] \ge 1- \sfe^{-\sum_{j \in \cS_i}x_j^*}  \ge 1-1/\sfe \ge (1-1/\sfe) \cdot y_i^*,
\]
which is due to Inequality~\eqref{ineq:ys} and the fact $y_i^* \le 1$. \tbf{Case (2)}: $\sum_{j \in \cS_i}x_j^* \le 1$. In this case,  $y_i^* \le \sum_{j \in \cS_i}x_j^* \le 1$. Observe that
\[
\E[Y_i^*] \ge 1- \sfe^{-\sum_{j \in \cS_i}x_j^*} \ge (1-1/\sfe) \cdot \bp{\sum_{j \in \cS_i}x_j^*} \ge  (1-1/\sfe) \cdot y_i^*,
\]
where the second inequality follows from the fact that function $(1-\sfe^{-x})/x$ is decreasing  over $x \in (0,1]$. Summarizing the two cases above, we conclude that $ \E[Y_i^*]\ge  (1-1/\sfe) \cdot y_i^*$. Note that the (expected) performance of $\rdr$ satisfies
\begin{align*}
\E[\rdr]&=\min_{g \in \cG} \E\bb{\sum_{i \in g} Y_i^*/|g|} \ge (1-1/\sfe) \cdot \min_{g \in \cG} \sum_{i \in g} y_i^* /|g|\\
&=(1-1/\sfe) \cdot \LP\eqref{obj-1} \ge (1-1/\sfe) \cdot \OPT,
\end{align*}
where \LP\eqref{obj-1} and \OPT denote the optimal value of \LP~\eqref{obj-1} and the performance of an optimal randomized strategy, and where the last inequality is due to Lemma~\ref{lem:lp}.
\end{proof}




\section{Experimental Results}
Our experiments involve quite a few datasets that were collected in the city of Chicago. 
For each dataset, we cite it as a reference pointing to the public link. Additionally, we offer more details in Appendix about how we use them in our experiments. All experiments are conducted on a PC with 2GHz Quad-Core Intel Core i7 processor and 8GB main memory.

\subsection{Experiment Setup}

\begin{figure}[t!]
  \centering
  {\includegraphics[width=0.7\columnwidth]{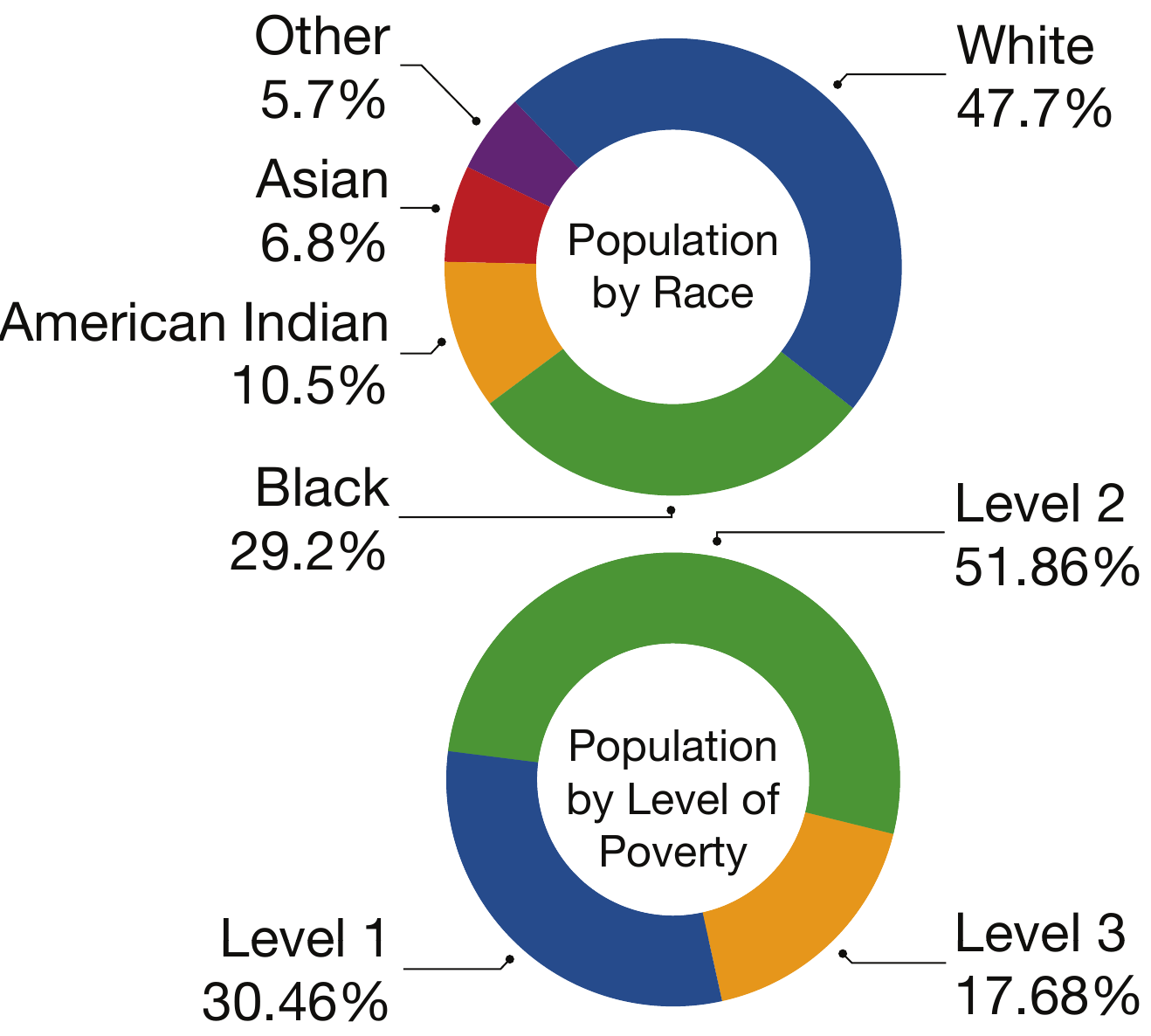}}
     \caption{The household distributions used in our experiments. \textbf{Top}: the distribution of households by races, \ie White, Black, American Indian, Asian, and some other ethic background. \textbf{Bottom}: the distribution of households by poverty levels.}
     \label{fig:household_dist}
\end{figure}

\xhdr{Data preprocessing}. 
We focus on needy households that are far away from the public transits, that is between 0.5 and 3.5 miles away form the nearest train stations and between 0.25 and 3.5 miles away from the nearest bus stations~\cite{degood2016can}. We identify a total number of 17,875 target households based on data from the Chicago Metropolitan Agency for Planning~\cite{data-cmap} and the public transportation data by the Chicago Data Portal~\cite{data-chicago}.  For each needy household, we set a poverty level following the guideline based on the income level estimation information provided by the Department of Health and Human Services~\cite{data-hhs}. 
We separate all target households into three groups: the first is at or above 200\% level of the poverty line (Group Level 1), the second group is between 185\% and  200\% of the line (Group Level 2), while the third is at or below 175\% (Group Level 3). A recent survey showed that the average cost for a ride-hailing trip is about \$25.37~\cite{averagecostperride}, and accordingly, we set the subsidy in the ride-hailing program as \$10, \$15, and \$20 per ride for Group Level 1,  2, and 3, respectively. We assign each household a random race following the distribution recorded by the official US Census in Chicago, 2022~\cite{data-wpr}; see Figure~\ref{fig:household_dist} (Top).

The set of candidate bus routes is generated as follows.  We cluster all needy households with the travel distance no more than 0.25 miles and then set a bus stop right in the center. After filtering out those impractical stations that are more than 3.5 miles away from any others, we have 649 candidate bus stops in total. Recall that each candidate bus line is for connecting residents with the nearest public transits. Thus, each route should end at either a metro station or a bus stop. Following this principle, we create 20 candidate bus routes with proper length, that is with 10-18 stops and no more than 0.75 miles between any two stops. For each bus route, we set the operating expense per vehicle revenue hour to \$140 based on information offered by the Chicago Transit Authority~\cite{data-cta}. 
We create two schedules for each bus route, one with the full-day and the other with the half-day working schedule.\footnote{We assume that the half-day working schedule can serve only half of the households on the route.}

\begin{figure*}[t!]
  \centering
  \begin{subfigure}[b]{0.32\linewidth}
    \includegraphics[width=\linewidth]{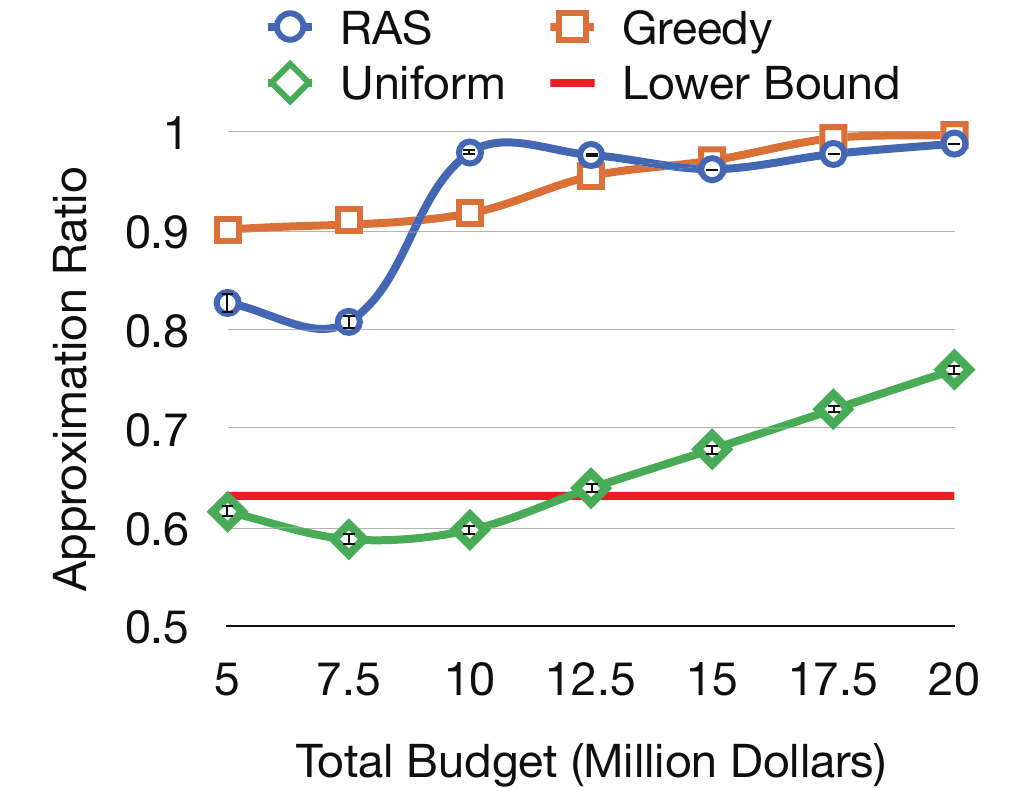}
    \caption{Approximation ratios with the bus-line-based program only.}
    \label{fig:ar_n}
  \end{subfigure}
  \hspace{1mm}
  \begin{subfigure}[b]{0.32\linewidth}
    \includegraphics[width=\linewidth]{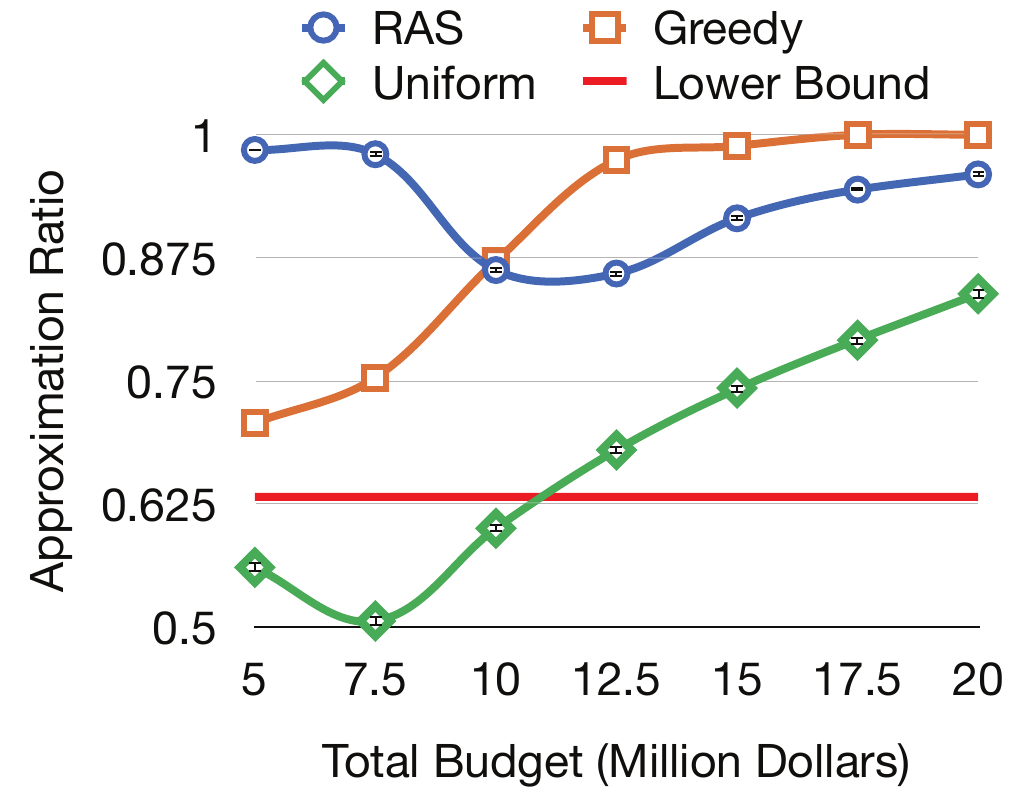}
    \caption{Approximation ratios when the ride-hailing-based program added.}
    \label{fig:ar_w}
  \end{subfigure}
  \hspace{1mm}
  \begin{subfigure}[b]{0.32\linewidth}
    \includegraphics[width=\linewidth]{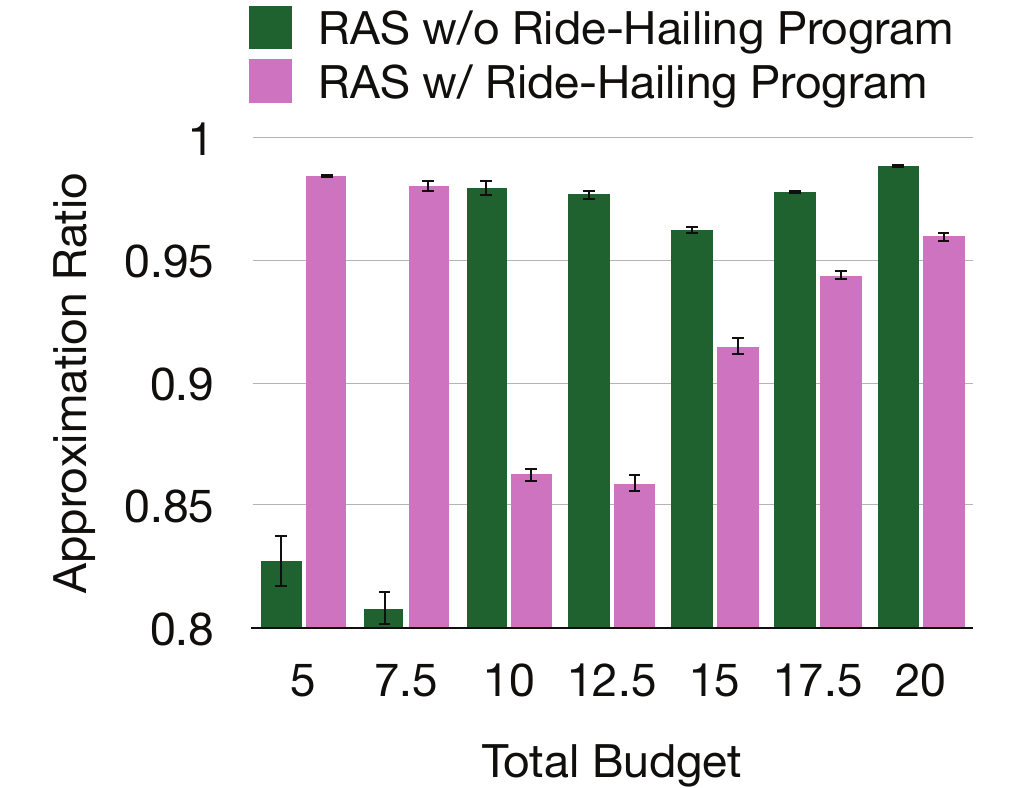}
    \caption{Comparison between before and after adding the ride-hailing-based program.}
    \label{fig:ar_wn}
  \end{subfigure}
  \caption{Experimental results on real datasets collected in Chicago: The total number of budgets $B$ takes values from $\{5,7.5,..., 20\}$ (million dollars).}
  \label{fig:ar_results}
\end{figure*}

\xhdr{Algorithms}. Suppose we have a total budget $B$ and a collection $J$ of candidate bus lines (which could potentially include those one-one virtual bus lines representing to cover a household via the ride-hailing program in case the latter is considered). In addition to the LP-based algorithm proposed in this paper (\rdr), we implement two baselines as follows. (a) \sgreedy: Always select the bus line $j \in J$ that maximizes the marginal gain of equity based on the current budget allocation (break ties arbitrarily), and repeat this procedure until the budget $B$ is exhausted or all target households are covered; (b) \srandom: Select a bus line from $J$ uniformly at random once at a time, and repeat this procedure until the budget $B$ is exhausted or all target households are covered;
As for randomized algorithms of \rdr and \srandom, we run each 1000 times and take the average as the final performance. We compare the performance of each algorithm against the optimal value to the benchmark~\LP~\eqref{obj-1} and compute that ratio as the final approximation ratio achieved. We also compute the 95\% confidence interval for each algorithm's performance to show the potential robustness.

\xhdr{Computational complexity of \rdr}. The running time of \rdr consists of two parts: Solving \LP~\eqref{obj-1} and the rounding procedure. Theoretically, the running time to solve~\LP~\eqref{obj-1} can be as low as $O^*(N^{2 + 1/6}\log(N / \delta))$~\citep{Cohen2019SolvingLP}, where $\delta$ is the relative accuracy and $N=|I|+|J|$ with $I$ and $J$ being the total number of households and bus lines, respectively. As for the rounding procedure, \rdr is guaranteed to eliminate at least one fractional value from the vector $\X$ of size $|J|$ each round and thus, it takes $O(|J|)$ time.  Therefore, theoretically, the dominant part of the running time is to solve the benchmark LP~\eqref{obj-1}.

\subsection{Results and Discussions} \label{sec:exp-res}
We run our experiments in two cases: the first is with bus-lines-based program only; the second is to add ride-hailing program and consider the two programs together. In our experiments, we aim to make a plan for one quarter with a total of budget of $B \in \{5,7.5, 10, \ldots, 20\}$ (million dollars).


Figure~\ref{fig:ar_n} shows results in the first case when the bus-lines-based program is the only choice. The approximation ratios of \rdr always stay above the theoretical lower bound of $1-1/\sfe \sim 0.632$ (plotted in red line) as shown in Theorem~\ref{thm:main-1}. As the total budget increases, all three algorithms' approximation ratios are increasing and approaching 1. This can be seen as follows: the more budget we have, the more capability each algorithm has to cover as many households as possible in each protected group, which leads to a higher equity as a result.  
\sgreedy can outperform  \rdr at some larger budgets since in that case, there is less need in optimizing the budget allocation among different approaches as \rdr is designed for.  Note that \sgreedy does not have any theoretical guarantees, which can perform arbitrary bad on some worst-scenario instance (since \sgreedy is deterministic, and as Example~\ref{exam:a} shows, any deterministic achieves an approximation ratio of zero).  

Figure~\ref{fig:ar_w} shows the results in the second case when both bus-lines-based and ride-hailing-based programs are considered. Our algorithm \rdr dominates the other two baselines when the total budget is capped by 10 million dollars. This highlights the superiority of the LP-based rounding algorithm when budget is limited, which is commonly observed in reality. As reported in~\cite{farcry}, ``the needed investment for public transit in U.S.\ is a far cry from the current numbers that have finally been passed.'' 

Figure~\ref{fig:ar_wn} demonstrates the high potentials of the ride-hailing program in promoting equity among protected groups when budget is relatively small. In the case of scarcity of public funds, the approach of covering needy households via ride-hailing trips prove to be far more efficient than the traditional one through opening new bus lines.  This is due to the fact that the ride-hailing-based approach enjoys more flexibility and a higher utilization rate of budget when the funds are insufficient. When budget becomes more abundant, it seems a better choice to consider the traditional approach only since the advantage of ride-hailing trips diminishes. 



\section{Conclusion}
In this paper, we propose  a theoretical model to promote equity in public transportation by optimizing budget allocation over  different approaches of improving the access to public transits for needy households. We design an LP-based rounding algorithm and prove that it achieves an optimal $1-1/\sfe$ approximation ratio. Additionally, we test our algorithm against a few natural baselines on real datasets, and experimental results confirm our theoretical predictions and highlight the effectiveness in promoting the social equity among households with low socioeconomic status.

Our work opens a few new directions. The first is to introduce more constraints to the model reflecting practical restrictions in the real world to make our model one step closer to reality. One example is to consider adding area-based capacity to the ride-hailing program, which is due to the limited availability of ride-hailing drivers in that area. We expect any newly added constraints could bring significant technical challenges in algorithmic design and analysis. The second is to parallelize the current algorithm such that it can be implemented and deployed more efficiently.

\section*{Acknowledgments}
Anik Pramanik and Pan Xu were partially supported by NSF CRII Award IIS-1948157. Pan Xu would like to thank Hui Kong and Xinyue Ye for stimulating discussions. 
The authors would like to thank the anonymous reviewers for their helpful feedback.
 \newpage

\bibliography{stable_ref}

 \onecolumn
 \appendix
\section{Proof of Property (\tbf{P4}) in Lemma~\ref{lem:pr}}
\begin{proof}
Consider a fixed set $\cS \subseteq J$ and a given rounding step. Let $\X$ and $\X'$ be the vector before and after the rounding step.  We split our discussion into the following cases. 

\tbf{Case (a)}: There is only one index of fractional value, say $p \in \cS$, which is selected into the rounding process. Then we have  
\begin{align*}
&\E\bb{\prod_{j \in \cS, j\neq p} (1-X_j) \cdot(1-X'_p)~~\big| ~~\X}  =\prod_{j \in \cS, j \neq p} \big(1-X_j\big) \cdot \E\big[(1-X_p') ~~\big| ~~ \X\big]=\prod_{j \in \cS, j \neq p} \big(1-X_j\big) \cdot (1-X_p)=\prod_{j \in \cS} \big(1-X_j\big).
\end{align*}

\tbf{Case (b)}: There are two indices of selected fractional values, say $p, q \in \cS$. Let $H:=\prod_{j \in (\cS-\{p,q\}) } \big(1-X_j\big)$.
\begin{align*}
\E\Big[\prod_{j \in \cS} (1-X'_j) ~~\big| ~~ \X \Big]& =H \cdot \E\Big[ (1-X'_p)\cdot  (1-X'_q)  ~~\big| ~~\X\Big]\\
&=H \cdot \bB{ \frac{\beta}{\alp+\beta}  \Big(1-X_p-\alp\Big) \cdot\Big(1-X_q+\frac{c_p \alp}{c_q}\Big)  +\frac{\alp}{\alp+\beta}  \Big(1-X_p+\beta\Big) \cdot\Big(1-X_q-\frac{c_p \beta}{c_q}\Big) }\\
&=H \cdot \bB{  \Big(1-X_p\Big) \cdot \Big(1-X_q\Big)-\frac{c_p}{c_q}\cdot (\alp \cdot\beta)} \le H \cdot  \Big(1-X_p\Big) \cdot \Big(1-X_q\Big)=\prod_{j \in \cS} \Big(1-X_j\Big).
\end{align*}

\tbf{Case (c)}: No index of any fractional values is selected in $\cS$. We observe that $X'_j=X_j$ for all $j \in \cS$.  Thus, we establish the claim that for any subset $\cS \subseteq J$, $\E[\prod_{j \in \cS}(1-X_j)]$ will never get increased.
\end{proof}

\section{Description of Datasets Relevant to Our Experiments}

\subsection{CMAP Data: \url{https://tinyurl.com/2cvwk4p5}}
The Chicago Metropolitan Agency for Planning (CMAP) conducted a comprehensive travel and activity survey for northeastern Illinois residents between August 2018 and April 2019 with a total of 12,391 households participating in the survey~\cite{data-cmap}. The survey contained detailed information about the household sizes, household income, household location coordinates, census county, vehicle ownership, trips frequency, trip destinations and other sociology-economic and travel inventory of each household. In this paper, we use the household location coordinates to calculate their distance from nearby transit hubs and locate the households that are deprived of public transportation.

\subsection{Household Income Data: \url{https://tinyurl.com/2p9mcnxy}}
The Department of Health and Human Services (HHS) has provided guidelines denoting the federal poverty thresholds using income level and household size~\cite{data-hhs}. Table 3 in the guidelines shows that part of the 2021 HHS poverty guidelines which are published in the Federal Register. Also in the Table, the column labels indicate different levels of poverty decreasing from left to right. The column with 100\% is the most poverty-stricken group and the 200\% column is the least poverty-stricken among the list. According to the Illinois Department of Human Services, 200\% poverty level is considered the base poverty line. In other words, income for a household must be at 200\% or below of the Federal Poverty Level (FPL) for the family to be eligible as poor.

\subsection{Operating Expense for Bus Lines: \url{https://tinyurl.com/mvkwznf6}}
Chicago Transit Authority (CTA) publishes yearly budget recommendations~\cite{data-cta} every year that reports the operating expenses, system revenue, funding, comparative statistics of other cities, future improvement and expenses, and other aspects of Chicago transportation. From the 2019 Budget Recommendations report of Tables ``Operating expense per vehicle revenue mile'' and ``Operating expense per vehicle revenue hour'', we can see that that the average cost to operate one vehicle per revenue mile is \$15.32 and the average cost to operate one vehicle per revenue hour is \$140.91.

	\end{document}